\def\ARXIVVERSION{1}
\documentclass{article}
\usepackage{iclr2027_conference,times}
\usepackage{etoolbox}

\iclrfinalcopy
\makeatletter
\patchcmd{\@maketitle}
  {\lhead{Published as a conference paper at ICLR 2027}}
  {\lhead{}}
  {}
  {\PackageError{softsat-arxiv}{Could not neutralize the ICLR publication header}{}}
\makeatother

\usepackage{amsmath,amsfonts,bm}

\def\eqref#1{equation~\ref{#1}}

\def\1{\bm{1}}

\DeclareMathAlphabet{\mathsfit}{\encodingdefault}{\sfdefault}{m}{sl}
\SetMathAlphabet{\mathsfit}{bold}{\encodingdefault}{\sfdefault}{bx}{n}

\usepackage{hyperref}
\hypersetup{hypertexnames=false}
\usepackage{url}
\usepackage{graphicx}

\usepackage{algorithm}
\usepackage{algorithmic}
\usepackage{amsmath}
\usepackage{amssymb}
\usepackage{amsthm}
\usepackage{booktabs}

\newtheorem{lemma}{Lemma}

\title{More Correct Mass, Worse Answers: Why Power Sampling Can Fail and How to Fix It}
\ifdefined\ARXIVVERSION
\author{
Haohui Yang\thanks{Equal contribution.} \&
Jiaxing Sun\footnotemark[1] \&
Xiujun Ma\thanks{Corresponding author.}\\[0.4em]
\small State Key Laboratory of General Artificial Intelligence, Peking University, Beijing, China\\
\small \texttt{\{yanghaohui,sunjiaxing\}@stu.pku.edu.cn}\\
\small \texttt{maxiujun@pku.edu.cn}
}
\date{}
\else
\author{Anonymous Authors}

\fi

\begin{document}

\maketitle

\begin{abstract}
Power Sampling sharpens a language model's distribution over complete generation trajectories, offering a verifier-free way to improve reasoning at inference time. It also has the potential to serve as a general-purpose front end for a broad range of downstream sampling methods. However, we uncover a striking paradox: Power Sampling can drive more probability mass toward correct trajectories while degrading the downstream inference it is intended to enhance. Using self-consistency as a representative case, we observe accuracy drops of up to 18.5 percentage points across models and reasoning benchmarks. We trace this paradox to two mismatches. Dose mismatch arises because a fixed exponent induces drastically different amounts of distributional change across problems. Coverage mismatch arises because global sharpening concentrates mass on a narrow set of dominant paths: high pass@k, often interpreted as evidence of preserved diversity, can therefore coexist with the loss of broad reasoning-path support required for downstream aggregation, search, and selection. Guided by this diagnosis, we replace uniform trajectory exponentiation with a deformation-controlled, support-preserving Power target that calibrates sharpening across problems while limiting the suppression of moderate-probability paths. In a same-budget instantiation with weighted self-consistency, the repaired sampler reverses the losses caused by global Power and outperforms standard multi-sample inference across reasoning benchmarks.
\end{abstract}


\section{Introduction}

Large language models have made substantial progress on complex tasks in mathematics, code generation, and scientific question answering \citep{lewkowycz2022quantitative,chen2021codex}. For much of this progress, the dominant recipe has been to scale the resources invested during training: model capacity, training data, and training compute \citep{kaplan2020scaling}. More recently, test-time scaling has emerged as a complementary axis for improving capability after training. Rather than modifying model parameters, it allocates additional inference compute to longer reasoning, search over possible solutions, or the generation of multiple candidate trajectories \citep{snell2025scaling,wu2025inference}. In doing so, it broadens the central scaling question from how to train a stronger model to how a trained model should spend computation when solving each problem.

Within this broader paradigm, repeated sampling offers a simple and broadly applicable strategy: it generates multiple reasoning trajectories and combines them through verification, selection, or aggregation \citep{chen2026modelswitch,khanh2026diverse}. Self-Consistency (SC), for example, selects the answer receiving the most support across independently sampled trajectories \citep{wang2023selfconsistency,oh2026latent}. Crucially, such methods rely on more than the presence of one strong trajectory. Their success also depends on whether the upstream distribution preserves the collective support on which the downstream decision is based \citep{tan2025selfconsistent,khanh2026diverse}.

This dependence makes the upstream trajectory distribution a central design choice. Power Sampling offers a verifier-free way to reshape that distribution: given a language model distribution $p$ over complete trajectories, it targets a sharpened distribution proportional to $p^\alpha$, favoring trajectories that the model already considers likely without additional training or a correctness verifier \citep{huang2025sharpening,karan2026reasoning}. Recent advances, including Scalable Power Sampling and Power-SMC, substantially reduce the cost of approximating this sequence-level target, making Power Sampling practical for large-scale inference \citep{ji2026scalable,azizi2026powersmc}. Because the target is independent of any particular downstream rule, it appears to be a natural front end for multi-trajectory reasoning.

\begin{figure*}[t]
    \centering
    \begin{minipage}[t]{0.242\textwidth}
        \centering
        \includegraphics[width=\linewidth]{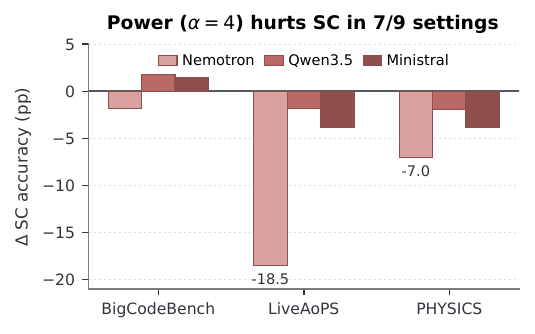}
        \vspace{-1mm}

        {\footnotesize\textbf{(a)} Downstream failure}
    \end{minipage}\hfill
    \begin{minipage}[t]{0.242\textwidth}
        \centering
        \includegraphics[width=\linewidth]{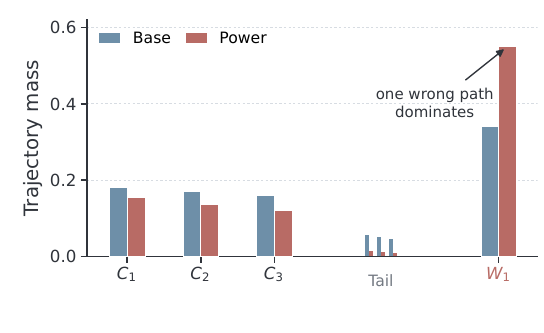}
        \vspace{-1mm}

        {\footnotesize\textbf{(b)} Coverage mismatch}
    \end{minipage}\hfill
    \begin{minipage}[t]{0.242\textwidth}
        \centering
        \includegraphics[width=\linewidth]{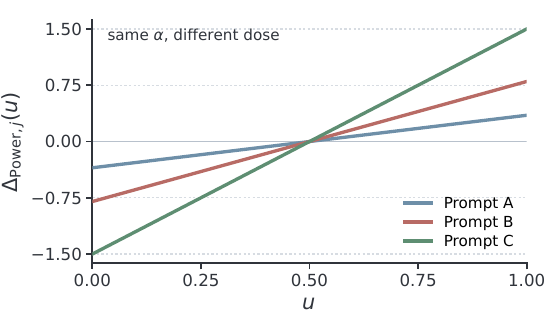}
        \vspace{-1mm}

        {\footnotesize\textbf{(c)} Dose mismatch}
    \end{minipage}\hfill
    \begin{minipage}[t]{0.242\textwidth}
        \centering
        \includegraphics[width=\linewidth]{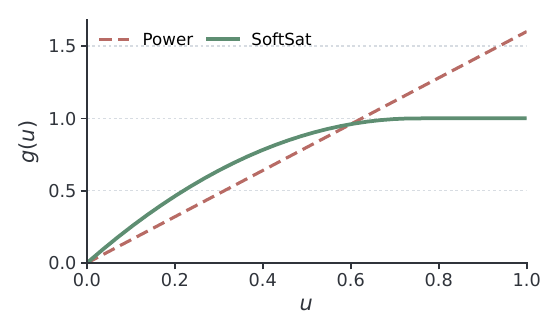}
        \vspace{-1mm}

        {\footnotesize\textbf{(d)} Relative-Rank SoftSat}
    \end{minipage}
    \caption{\textbf{Power sharpening can conflict with multi-trajectory aggregation, motivating a calibrated and support-preserving alternative.} Here $g$ denotes the additive log-weight gain in $q\propto p\exp(g)$.
    \textbf{(a)} Change in SC accuracy when Base SC8 is replaced by Power ($\alpha=4$) across three models and three benchmarks; seven of nine settings degrade.
    \textbf{(b)} Blue and red bars show Base and illustratively sharpened trajectory masses. $C_1$--$C_3$ support the same correct answer, $W_1$ supports a competing wrong answer, and Tail groups other low-mass paths; sharpening makes $W_1$ dominant.
    \textbf{(c)} Curves plot each prompt's centered Power gain $\Delta_{\mathrm{Power},j}(u)$ against the common within-prompt rank coordinate $u$ under a shared $\alpha$.
    \textbf{(d)} Repair schematic. On the same $u$, dashed red and solid green curves show a linearly increasing Power reference $g(u)$ and the saturating Relative-Rank SoftSat gain, respectively.}
    \label{fig:main-paradox}
\end{figure*}

However, this seemingly natural composition exposes a fundamental tension: assigning more probability to individually correct trajectories does not guarantee a better aggregated answer. Indeed, Power Sampling can move \emph{more} probability mass onto correct trajectories while weakening the answer-level support on which aggregation relies, producing \emph{worse} final answers. The downstream consequences are substantial. Across nine model--benchmark settings, fixed-exponent Power reduces downstream scores in seven, with a largest drop of 18.5 percentage points; notably, the degradation appears in all six LiveAoPSBench and PHYSICS settings shown in Figure~\ref{fig:main-paradox}(a). Thus, trajectory-level sharpening is not reliably aligned with the multi-trajectory decision it is intended to support.

Why does this misalignment arise? We identify two mismatches between global sharpening and downstream aggregation. First, \emph{coverage mismatch} occurs because Power rewards trajectories individually, whereas aggregation depends on their collective answer-level support. Figure~\ref{fig:main-paradox}(b) gives a minimal example: $C_1$--$C_3$ are distinct moderate-probability trajectories that map to the same correct answer, while $W_1$ is the individually strongest trajectory and maps to a competing wrong answer. Under Base, the correct answer wins through the combined support of $C_1$--$C_3$. Because Power exponentiates each trajectory before answer-level aggregation, it magnifies $W_1$'s individual advantage and suppresses the moderate correct paths; after renormalization, $W_1$ becomes dominant and the aggregated decision can flip. The low-mass Tail emphasizes that this failure is a loss of broadly distributed support, rather than the disappearance of every alternative path. Consequently, high pass@$k$ certifies that a correct trajectory remains reachable, but not that the distribution preserves the effective support required by aggregation, search, or selection \citep{karan2026reasoning,yue2025reasoningcapacity,tan2025selfconsistent,khanh2026diverse}. Second, \emph{dose mismatch} occurs because a fixed exponent applies the same algebraic transformation, not the same amount of distributional change. Since likelihood geometry varies across prompts, the same $\alpha$ can produce a mild adjustment on one problem but near-collapse on another. Figure~\ref{fig:main-paradox}(c) places prompts on the common rank coordinate $u$: prompt-specific likelihood spreads turn the same $\alpha$ into different centered gain slopes $\Delta_{\mathrm{Power},j}(u)$. Together, coverage mismatch explains why sharpening can damage aggregation, while dose mismatch explains why that damage is uneven and difficult to control.

This diagnosis suggests that a successful repair must control both the shape and strength of sharpening. We therefore introduce \emph{Relative-Rank SoftSat}. To address coverage mismatch, SoftSat promotes moderately ranked trajectories while saturating gains at the top, preventing a few dominant paths from absorbing disproportionate mass. To address dose mismatch, relative rank replaces raw likelihood with a common within-prompt coordinate in $[0,1]$, making the strength of reweighting more comparable across problems. Together, the two components favor strong trajectories without discarding the broader support required by downstream aggregation. Figure~\ref{fig:main-paradox}(d) makes this design response explicit through $g(u)$: on the same relative-rank coordinate, the dashed linear Power reference keeps growing, whereas SoftSat raises moderate ranks and saturates toward the top.

Finally, we exploit the candidate-level parallelism enabled by modern Power samplers \citep{ji2026scalable,azizi2026powersmc}. Rather than invoking Power-SMC independently $N$ times for $\mathrm{SC}_N$, we recast target sampling as importance-weighted aggregation over a shared Base candidate pool. The resulting estimator reuses the same $N$ trajectories as Base+$\mathrm{SC}_N$, while our analysis shows that it asymptotically preserves the decision of direct target sampling under a standard positive-margin condition. This same-budget realization therefore retains the computational appeal of parallel sampling while allowing the repaired target to guide downstream aggregation.

Our contributions are threefold:
\begin{enumerate}
    \item We show that Power Sampling can increase correct-trajectory mass yet degrade multi-trajectory decisions, and identify coverage mismatch and dose mismatch as the causes.
    \item We introduce Relative-Rank SoftSat, which replaces absolute likelihood with within-prompt rank to control dose and saturates gains on dominant trajectories to preserve coverage.
    \item We connect ideal target sampling to importance-weighted aggregation over a shared Base pool, enabling a low-cost realization that asymptotically preserves the target decision.
\end{enumerate}

\section{Related Work}

\subsection{Sequence-Level Distribution Sharpening}

Distribution sharpening reallocates probability toward outputs already favored by a model and has been studied as a mechanism underlying language-model self-improvement \citep{huang2025sharpening}. Within this view, Power Sampling operates at inference time by exponentiating complete-trajectory probabilities and approximating the resulting target with autoregressive Markov chain Monte Carlo \mbox{\citep{karan2026reasoning}}. Because iterative MCMC introduces substantial sequential latency, subsequent work has focused on making this target practical at inference time. Specifically, Scalable Power Sampling replaces iterative MCMC with an autoregressive approximation based on sampled continuations \citep{ji2026scalable}, whereas Power-SMC uses sequential importance weighting and resampling to advance a particle population in a GPU-friendly batched decode \citep{azizi2026powersmc}. By reducing the sequential bottleneck, these methods make fixed-exponent Power distributions substantially easier to sample from. Their primary focus, however, remains how to obtain Power samples efficiently; our work instead asks how the resulting target and parallel candidate computation should interact with a downstream multi-sample decision rule.

Despite this progress, existing evaluations largely characterize the resulting sampler through single-output accuracy and pass@$k$ \citep{karan2026reasoning,ji2026scalable}. Yet pass@$k$ establishes only whether a correct trajectory remains reachable; it does not describe how probability mass is distributed across reasoning paths or answer modes. This distinction matters for multi-sample inference, where limited output diversity can constrain Best-of-$N$ reasoning \citep{khanh2026diverse} and repeated generations may form a stable but incorrect consensus \citep{tan2025selfconsistent}. Accordingly, our work studies the downstream compatibility of sharpening and modifies the target deformation to preserve the support required by aggregation.

\subsection{Multi-Sample Consistency and Aggregation}

Multi-sample aggregation methods can be organized around two choices: how they define agreement and how much support each candidate contributes. The first choice determines which trajectories are treated as evidence for the same outcome. Self-Consistency groups reasoning trajectories by their normalized final answer and selects the largest group \citep{wang2023selfconsistency}, whereas CodeT extends consistency to programs by jointly evaluating generated code and tests through execution \citep{chen2023codet}. For outputs without a canonical short answer, Universal Self-Consistency asks an LLM to identify the most consistent response \citep{chen2023universal}; Latent Self-Consistency learns representations for majority-set selection across short- and long-form answers \citep{oh2026latent}; and ModeX identifies a dominant response through a similarity graph over complete generations \citep{choi2026modex}. Although these methods use different equivalence relations, all of their readouts depend on collective support across sampled candidates.

The second choice determines how much support each candidate contributes. Whereas standard SC assigns one vote per trajectory, CISC weights answers using model-reported confidence \citep{taubenfeld2025confidence}, RASC evaluates both rationale quality and answer consistency \citep{wan2025rasc}, and RPC combines answer frequency with internal sequence probability \citep{zhou2025bridging}. Dynamic distributional alignment instead adjusts token-level sampling temperature in response to the evolving answer distribution \citep{li2025dynamic}. Other approaches improve sampling efficiency by adapting the sample count or switching among complementary models \citep{wang2026optscale,chen2026modelswitch}. In contrast to confidence or correctness scores, our weights represent the probability mass assigned by an explicit target trajectory distribution. Consequently, that mass can be accumulated under answer equality, execution-based relations, or semantic similarity without changing the task-specific definition of consistency.

\section{Preliminaries}
\label{sec:preliminaries}

\subsection{Trajectory Distributions and Power Sampling}

Given a prompt, let $\tau=(\tau_1,\ldots,\tau_T)$ denote a complete generation trajectory; we omit prompt conditioning for clarity. The Base model assigns trajectory probability and log-likelihood
\begin{equation}
\begin{aligned}
p(\tau)&=\prod_{t=1}^{T}p(\tau_t\mid\tau_{<t}),\\
\ell(\tau)&=\log p(\tau)=\sum_{t=1}^{T}\log p(\tau_t\mid\tau_{<t}).
\end{aligned}
\label{eq:base-trajectory}
\end{equation}
Token-level temperature sampling modifies each next-token distribution separately. At temperature $\gamma>0$, it induces
\begin{equation}
p_{\gamma}(\tau)
=\prod_{t=1}^{T}
\frac{p(\tau_t\mid\tau_{<t})^{1/\gamma}}
{\sum_v p(v\mid\tau_{<t})^{1/\gamma}}.
\label{eq:token-temperature}
\end{equation}
Lowering $\gamma$ sharpens every prefix-conditioned decision. Power Sampling instead sharpens the distribution over complete trajectories. For $\alpha>1$, its ideal target is
\begin{equation}
q_\alpha(\tau)=\frac{p(\tau)^\alpha}{Z_\alpha},
\qquad
Z_\alpha=\sum_{\tau'}p(\tau')^\alpha.
\label{eq:power-target}
\end{equation}
The token-temperature distribution in \eqref{eq:token-temperature} and the Power target in \eqref{eq:power-target} are generally different because token-level normalization depends on the generated prefix. Power-SMC approximates the sequence-level target with $M$ particles and returns one approximate Power trajectory per run \citep{azizi2026powersmc}. Our analysis uses the ideal $q_\alpha$ as its reference and separates this target from finite-$M$ approximation error. The number of trajectories used by a downstream aggregation rule is denoted by $N$; it is distinct from the internal particle count $M$.

\subsection{From Trajectories to Consensus}

Let $A(\tau)$ be the normalized answer extracted from $\tau$, and let $R(\tau)\in\{0,1\}$ indicate whether the trajectory is correct. Any trajectory distribution $\mu$ induces answer mass and total correct-trajectory mass
\begin{equation}
Q_\mu(a)=\Pr_{\tau\sim\mu}[A(\tau)=a],
\qquad
C_\mu=\mathbb{E}_{\tau\sim\mu}[R(\tau)].
\label{eq:outcome-mass}
\end{equation}
When correctness is determined by a unique gold answer $a^\star$, $C_\mu=Q_\mu(a^\star)$; for code, correctness may instead be determined by execution. Given $N$ independently sampled trajectories, Self-Consistency selects
\begin{equation}
\widehat a_{\mathrm{SC}}
=\arg\max_a\sum_{i=1}^{N}\mathbf{1}\{A(\tau_i)=a\}.
\label{eq:self-consistency}
\end{equation}
For executable or open-ended outputs without a canonical answer string, ModeX instead derives consensus from a similarity graph over complete responses \citep{choi2026modex}. These readouts use different notions of agreement, but both convert support across multiple trajectories into a final decision.

\section{Mismatch Analysis}
\label{sec:mismatch}

A higher chance of sampling a correct trajectory need not lead to a better decision from multiple samples. To understand this gap, we ask two questions that pass@$k$ leaves unanswered: where does Power move probability mass within a problem, and does the same exponent move different problems by comparable amounts? These questions expose two complementary failure modes, which we call \emph{coverage mismatch} and \emph{dose mismatch}, as summarized in Figure~\ref{fig:mismatch-overview}. The figure uses one representative setting; the appendix reports the same diagnostics across all three evaluated models.

\begin{figure*}[t]
    \centering
    \includegraphics[width=0.98\textwidth]{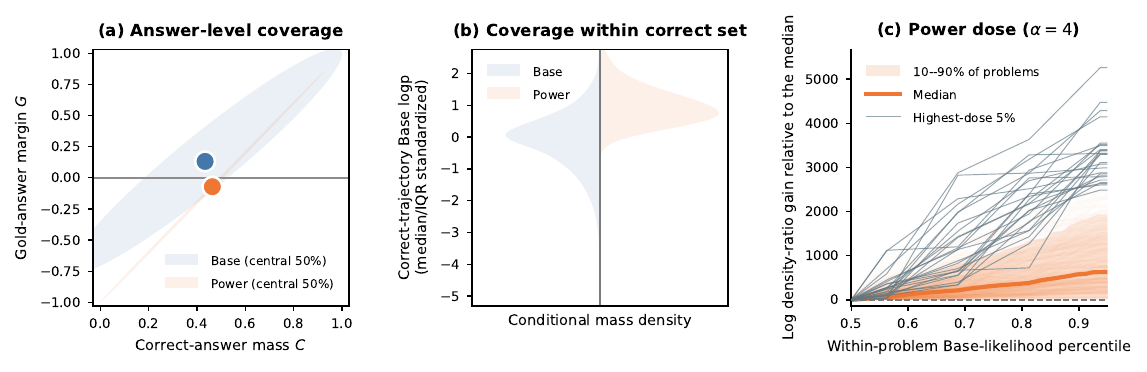}
    \caption{Three views of why Power sharpening need not improve downstream aggregation in a representative setting. (a) Power increases correct-answer mass $C$ while reducing the gold-answer margin $G$; filled markers denote across-problem centers and translucent ellipses summarize the central 50\% of problems. (b) Conditional on correctness, Power concentrates mass in a narrower, high-Base-likelihood region; the two densities are mirrored only for visual comparison. (c) A fixed $\alpha$ produces widely different within-problem log-weight gains; the shaded region spans the 10th--90th percentiles across problems, and blue-gray curves mark the 5\% of problems with the largest gain at the 90th likelihood percentile.}
    \label{fig:mismatch-overview}
\end{figure*}

\subsection{Coverage Mismatch}
\label{sec:coverage-mismatch}

Pass@$k$ captures whether at least one correct trajectory appears among $k$ samples. For independent draws from a trajectory distribution $\mu$,
\begin{equation}
\operatorname{pass@}k(\mu)
=1-\left(1-C_\mu\right)^k.
\label{eq:pass-k}
\end{equation}
Because this quantity depends only on $C_\mu$, it says nothing about how correct mass is divided among trajectories or how incorrect mass is organized across competing answers. A high pass@$k$ can therefore coexist with a smaller answer margin $G_\mu$.

Figure~\ref{fig:mismatch-overview}(a) contrasts total correct mass with the correct answer's margin over its strongest competitor. Under Power, the problem distribution shifts toward larger correct mass but a smaller margin: the strongest incorrect answer gains still more mass. The answer-level view nevertheless leaves open where the additional correct mass goes. Figure~\ref{fig:mismatch-overview}(b) resolves this mass along within-problem standardized Base log-likelihood, with problems weighted equally. Power concentrates it in a narrower, higher-likelihood region of the correct set.

The two panels expose different aspects of coverage: panel (a) concerns competition between answers, whereas panel (b) concerns concentration within the correct set. The latter motivates examining response-level readouts such as ModeX, but does not by itself establish a change in response modes or in the winner of answer-based SC. Together, the two views show why total correct mass alone is insufficient to characterize the support available to aggregation.

\subsection{Dose Mismatch}
\label{sec:dose-mismatch}

Coverage mismatch concerns where mass moves. A separate failure arises because a fixed $\alpha$ does not impose a comparable deformation across problems. Power acts on raw sequence-level likelihood rather than a problem-relative scale. Its relative reweighting of two trajectories satisfies
\begin{equation}
\log
\frac{p_j(\tau_i)^{\alpha-1}}
     {p_j(\tau_k)^{\alpha-1}}
=(\alpha-1)\bigl[\ell_j(\tau_i)-\ell_j(\tau_k)\bigr].
\label{eq:power-relative-logweight}
\end{equation}
Although $\alpha$ is shared across problems, the likelihood gap is not. Narrow gaps yield mild reweighting, whereas wide gaps can make the same exponent sharply concentrate the distribution.

To place these deformations on a common trajectory coordinate, let $L_j(u)$ be the $u$-quantile of Base log-likelihood for problem $j$. The Power log-weight gain at that position relative to the within-problem median is
\begin{equation}
\Delta_{\mathrm{Power},j}(u)
=(\alpha-1)\bigl[L_j(u)-L_j(1/2)\bigr].
\label{eq:power-dose-curve}
\end{equation}
This aligns trajectories by their within-problem position while retaining the actual log-weight gain applied by Power.

This heterogeneity is not confined to selected trajectory pairs. If $\mathbb{E}_{\tau\sim p_j}[|\ell_j(\tau)|^3]<\infty$, then as $\alpha\downarrow1$, the full target obeys
\begin{equation}
D_{\mathrm{KL}}\!\left(q_{\alpha,j}\,\|\,p_j\right)
=\frac{(\alpha-1)^2}{2}
\operatorname{Var}_{\tau\sim p_j}[\ell_j(\tau)]
+O((\alpha-1)^3).
\label{eq:local-power-kl}
\end{equation}
Thus, problem-specific likelihood dispersion controls the leading-order movement of the entire Power target. We defer the exact identity and derivation to the supplementary material. Figure~\ref{fig:mismatch-overview}(c) makes the consequence visible: the curves fan outward under a fixed $\alpha$, with a small set of problems receiving markedly larger gains.

Coverage mismatch concerns where mass moves; dose mismatch concerns how far the distribution moves. Power couples both effects through a single global exponent, motivating a repair that controls them separately.

\section{Method}
\label{sec:method}

The preceding analysis points to two changes: likelihood must first be placed on a comparable within-problem scale, and the gain assigned on that scale must preserve useful coverage. Before defining these changes, we establish how a sequence-level target can be combined with a multi-sample readout without repeatedly running its sampler.

\subsection{From Power Sampling to Weighted Consensus}

A single Power-SMC run uses $M$ particles to return one approximate Power trajectory. A literal Power-SMC+$\mathrm{SC}_N$ composition therefore requires $N$ independent runs. We instead estimate the answer mass of the ideal Power target from one pool of $N$ Base trajectories. By \eqref{eq:power-target}, its unnormalized importance ratio is
\begin{equation}
r_\alpha(\tau)
=p(\tau)^{\alpha-1}
=\exp\!\left((\alpha-1)\ell(\tau)\right).
\label{eq:power-importance-ratio}
\end{equation}
For independent $X_{1:N}\sim p$ and $Y_{1:N}\sim q_\alpha$, consider
\begin{equation}
\begin{aligned}
\widehat Q_N^{\mathrm{weight}}(a)
&=\frac{\sum_{i=1}^{N}r_\alpha(X_i)\mathbf{1}\{A(X_i)=a\}}
{\sum_{i=1}^{N}r_\alpha(X_i)},\\
\widehat Q_N^{\mathrm{sample}}(a)
&=\frac{1}{N}\sum_{i=1}^{N}\mathbf{1}\{A(Y_i)=a\}.
\end{aligned}
\label{eq:two-answer-mass-estimators}
\end{equation}

\begin{lemma}[Asymptotic expectation equivalence]
\label{lem:power-expectation-equivalence}
For every fixed answer $a$,
\begin{equation}
\begin{aligned}
\lim_{N\to\infty}\mathbb{E}\!\left[\widehat Q_N^{\mathrm{weight}}(a)\right]
&=\lim_{N\to\infty}\mathbb{E}\!\left[\widehat Q_N^{\mathrm{sample}}(a)\right]\\
&=Q_{q_\alpha}(a).
\end{aligned}
\end{equation}
\end{lemma}

\begin{lemma}[Approximation rate]
\label{lem:power-estimation-error}
If $0<q_\alpha(\tau)/p(\tau)\leq C$ almost surely for some finite $C$, then
\begin{equation}
\begin{aligned}
\mathbb{E}\!\left[\left(
\widehat Q_N^{\mathrm{weight}}(a)-\widehat Q_N^{\mathrm{sample}}(a)
\right)^2\right]&=O(N^{-1}),\\
\widehat Q_N^{\mathrm{weight}}(a)-\widehat Q_N^{\mathrm{sample}}(a)
&=O_p(N^{-1/2}).
\end{aligned}
\end{equation}
\end{lemma}

The proofs, including a finite-sample bound, are given in the appendix. If the answer space is finite and $Q_{q_\alpha}$ has a unique mode separated by a positive margin, both estimators consequently select the same mode with probability tending to one. This connection is deliberately made at the ideal-target level: it does not couple the result to the finite-$M$ approximation error of Power-SMC, nor does it turn a deterministically weighted Base pool into fresh target samples.

\subsection{Relative-Rank SoftSat}

Power reads the absolute gap in trajectory log-likelihood. To control dose across problems, we replace it with the within-problem rank
\begin{equation}
u_p(\tau)=\Pr_{X\sim p}\!\left[\ell(X)\leq\ell(\tau)\right]\in[0,1].
\label{eq:relative-rank}
\end{equation}
This coordinate depends only on a trajectory's likelihood ordering within the same problem. We then shape its gain with
\begin{equation}
\operatorname{SoftSat}_m(u)
=1-\left(1-\min\left\{\frac{u}{m},1\right\}\right)^2,
\qquad 0<m\leq1,
\label{eq:softsat}
\end{equation}
and define the corresponding reweighting multiplier
\begin{equation}
r_{\beta,m}(\tau)
=\exp\!\left[
\beta\operatorname{SoftSat}_m\!\left(u_p(\tau)\right)
\right].
\label{eq:softsat-importance-ratio}
\end{equation}
The resulting target is $q_{\beta,m}(\tau)\propto p(\tau)r_{\beta,m}(\tau)$. The importance ratios in \eqref{eq:power-importance-ratio} and \eqref{eq:softsat-importance-ratio} expose the repair directly in the weighting view. Power reads the raw coordinate $\ell(\tau)$, so the same coefficient $\alpha-1$ inherits problem-specific likelihood gaps; replacing it with $u_p(\tau)$ addresses dose mismatch. Power also applies a linear log-weight gain across that coordinate, whereas SoftSat raises the gain for moderate ranks and saturates it for $u\geq m$; this change of shape addresses coverage mismatch. Because the SoftSat potential lies in $[0,1]$, the target-to-Base density ratio is bounded, so the preceding approximation applies to this target as well.

\subsection{Finite-Pool Weights and Consensus Readout}

Given a Base pool $\tau_{1:N}$, we estimate the relative rank in
\eqref{eq:relative-rank} by its empirical counterpart,
\begin{equation}
\widehat u_i
=\frac{1}{N}\sum_{k=1}^{N}
\mathbf{1}\{\ell(\tau_k)\leq\ell(\tau_i)\}.
\label{eq:empirical-softsat-rank}
\end{equation}
We obtain the finite-pool multiplier by clipping the shaped gain,
\begin{equation}
\widetilde r_i
=\operatorname{clip}_{[1-\eta,\,1+\kappa\eta]}
\!\left(
\exp\!\left[\beta\operatorname{SoftSat}_m(\widehat u_i)+c\right]
\right).
\label{eq:bounded-softsat-ratio}
\end{equation}
The offset $c$ makes these multipliers mean one; their normalized values are
used as candidate weights $W_{1:N}$.

The weighted candidates are passed to a task-appropriate consensus readout
$\mathcal{D}$. For answer-based tasks, this gives weighted Self-Consistency,
\begin{equation}
\mathcal{D}_{\mathrm{SC}}(\tau_{1:N},W_{1:N})
=\arg\max_a\sum_{i=1}^{N}W_i\mathbf{1}\{A(\tau_i)=a\}.
\label{eq:weighted-sc}
\end{equation}
For free-form responses, we apply the same candidate weights throughout the
ModeX readout. Uniform weights recover the original readout in either case.

\begin{algorithm}[H]
\caption{Relative-Rank SoftSat with a Weighted Consensus Readout}
\label{alg:relative-rank-softsat}
\begin{algorithmic}[1]
\REQUIRE Prompt $x$, Base model $p_\theta$, budget $N$, readout $\mathcal{D}$, parameters $(m,\beta,\eta,\kappa)$
\ENSURE Consensus output $\widehat y$
\STATE Sample $\tau_{1:N}\sim p_\theta(\cdot\mid x)$
\STATE Rank $\ell(\tau_{1:N})$ within the pool to obtain $\widehat u_{1:N}$
\STATE Apply SoftSat and bounded mean-one rescaling to obtain $\widetilde r_{1:N}$
\STATE Normalize $W_i\leftarrow\widetilde r_i/\sum_j\widetilde r_j$
\RETURN $\widehat y\leftarrow\mathcal{D}(\tau_{1:N},W_{1:N})$
\end{algorithmic}
\end{algorithm}

Our method reuses the same $N$ Base trajectories
as uniform consensus, and therefore incurs the same model-generation cost with
no additional model calls.

\section{Experiments}

We evaluate whether Power's failure persists across response spaces and whether
SoftSat mitigates it without changing the Base candidate pool or generation
budget.

\subsection{Experimental Setup}

\paragraph{Benchmarks and models.}
Our evaluation spans code, mathematical reasoning, and physics. We use the Full-Instruct split of BigCodeBench (1,140 problems)~\citep{zhuo2025bigcodebench}, the April 2024 slice of LiveAoPSBench (498 problems)~\citep{mahdavi2025liveaops}, and the text-only Regular subset of the PHYSICS official test split (503 problems)~\citep{feng2025physics}. We evaluate three models in the 4--9B range: NVIDIA-Nemotron-3-Nano-4B-BF16, Qwen3.5-9B, and Ministral-3-8B-Reasoning-2512.

\paragraph{Generation and readout.}
For every problem, we draw eight Base trajectories at temperature 1 with top-$k=50$, using the benchmark prompt, the model's chat template, and a maximum generation length of 16,384 tokens. All multi-sample methods reuse this pool. \emph{Uniform} assigns equal mass to all trajectories, while \emph{Power} uses $\alpha=4$. \emph{SoftSat} uses one configuration across all model--benchmark pairs: $m=0.75$, $\beta=1$, $\eta=0.25$, and $\kappa=64$, giving clip bounds $[0.75,17]$. Power and SoftSat therefore denote deterministic weighted readouts of the same Base pool, rather than fresh samples from either target distribution.

We use a readout matched to each response space. BigCodeBench uses ModeX~\citep{choi2026modex} followed by the official execution evaluator. LiveAoPSBench and PHYSICS group equivalent extracted answers and select the group with greatest uniform or weighted mass; scoring follows their official equivalence procedures. We report each benchmark's official score and do not average scores across benchmarks.

\subsection{Main Results}

\begin{table*}[t]
\centering
\small
\setlength{\tabcolsep}{5pt}
\caption{Official benchmark scores (\%) with eight Base trajectories per problem. \emph{Single} uses the first sampled trajectory; the remaining columns apply task-specific consensus to the same pool. $\Delta$ is SoftSat minus Uniform. Bold marks the best multi-sample result in each row.}
\label{tab:main-results}
\begin{tabular}{@{}llrrrrr@{}}
\toprule
Benchmark & Model & Single & Uniform & Power ($\alpha{=}4$) & SoftSat ($\beta{=}1$) & $\Delta$ \\
\midrule
BigCodeBench & Nemotron-4B & 17.456 & 24.386 & 22.544 & \textbf{24.561} & +0.175 \\
             & Qwen3.5-9B  & 34.386 & 36.579 & \textbf{38.333} & 37.719 & +1.140 \\
             & Ministral-8B & 29.123 & 33.421 & 34.825 & \textbf{34.912} & +1.491 \\
\midrule
LiveAoPSBench & Nemotron-4B & 21.486 & \textbf{34.137} & 15.663 & 33.133 & $-1.004$ \\
              & Qwen3.5-9B  & 59.237 & \textbf{64.659} & 62.851 & \textbf{64.659} & 0.000 \\
              & Ministral-8B & 44.177 & 50.201 & 46.386 & \textbf{51.606} & +1.405 \\
\midrule
PHYSICS & Nemotron-4B & 18.869 & \textbf{23.891} & 16.892 & 23.636 & $-0.255$ \\
        & Qwen3.5-9B  & 69.422 & \textbf{72.913} & 70.979 & 72.822 & $-0.092$ \\
        & Ministral-8B & 35.957 & 43.164 & 39.332 & \textbf{43.686} & +0.522 \\
\bottomrule
\end{tabular}
\end{table*}

Table~\ref{tab:main-results} shows that Power is not a reliable alternative to uniform consensus. It underperforms Uniform in all six LiveAoPSBench and PHYSICS settings, by 1.808--18.474 points on LiveAoPSBench and 1.934--6.999 points on PHYSICS. On BigCodeBench, in contrast, Power improves Qwen3.5 and Ministral but degrades Nemotron. Its effect therefore varies substantially across response spaces and models.

SoftSat is substantially more stable. Relative to Uniform, it improves five of the nine settings, ties one, and has a largest regression of 1.004 points. On the two largest Power failures, both with Nemotron, the gaps to Uniform fall from 18.474 to 1.004 points on LiveAoPSBench and from 6.999 to 0.255 points on PHYSICS. SoftSat does not uniformly outperform Uniform; the main observation is instead that it removes Power's large regressions while retaining gains in five settings.

\subsection{Budget-Dependent Strength}

The appropriate sharpening strength changes with the number of available trajectories. Figure~\ref{fig:budget-beta} reports a budget sweep for Ministral on BigCodeBench. The best $\beta$ rises from $0.25$ at $N=4$ to $2$ at $N=32$, while SoftSat improves over uniform SC at each evaluated budget. Thus, a larger candidate pool can support a stronger reweighting rule, but the resulting accuracy must still be read against the corresponding uniform-consensus baseline.

\begin{figure*}[t]
    \centering
    \includegraphics[width=0.93\textwidth]{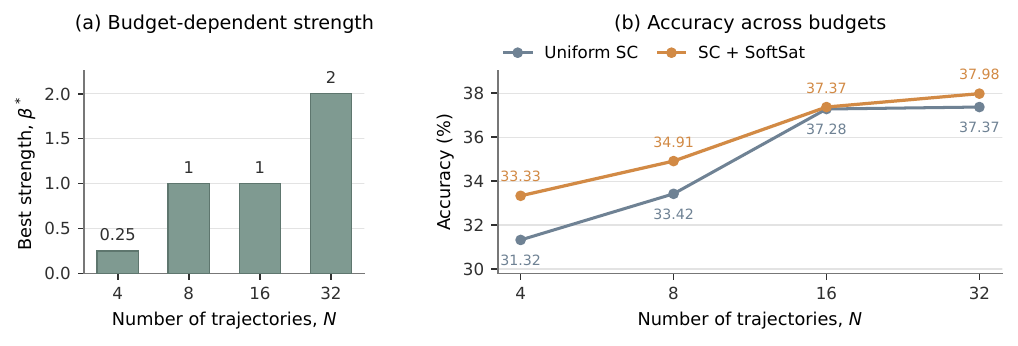}
    \caption{Budget-dependent SoftSat behavior for Ministral on BigCodeBench. (a) The best sharpening strength $\beta^*$ for each trajectory budget $N$. (b) Accuracy of uniform $\mathrm{SC}_N$ and SoftSat-weighted $\mathrm{SC}_N$ at the corresponding $\beta^*$. Results for all three benchmarks appear in the appendix.}
    \label{fig:budget-beta}
\end{figure*}

\subsection{Repairing the Diagnosed Mismatches}

The accuracy gains in Table~\ref{tab:main-results} do not by themselves show that SoftSat addresses the failure identified in Section~\ref{sec:mismatch}. Figure~\ref{fig:mismatch-repair} therefore returns to the same representative setting as Figure~\ref{fig:mismatch-overview} and overlays Base, Power, and SoftSat. At the answer level, Power moves the across-problem center from $(C,G)=(0.433,0.132)$ to $(0.464,-0.072)$, whereas SoftSat gives $(0.439,0.133)$ and retains a coverage profile close to Base. Within the correct set, Power shifts conditional mass toward the highest-Base-likelihood trajectories, while SoftSat remains much closer to the Base profile.

\begin{figure*}[t]
    \centering
    \includegraphics[width=0.88\textwidth]{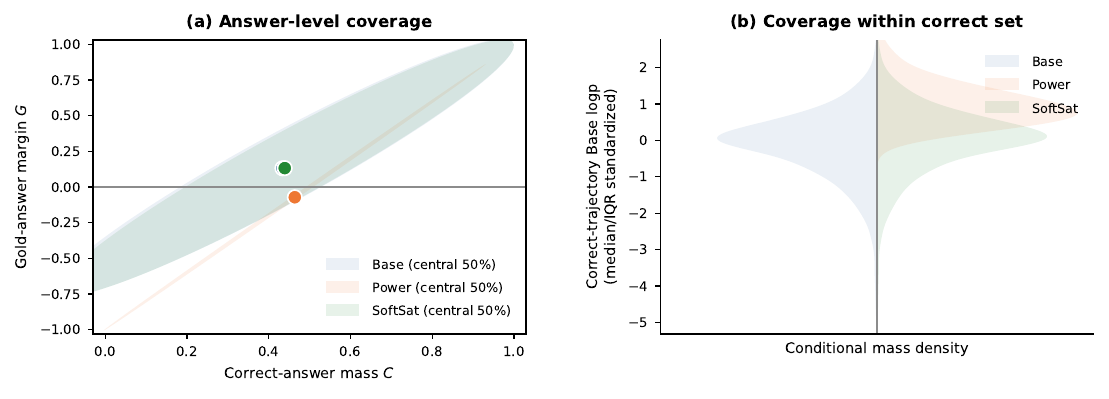}
    \caption{SoftSat repairs the coverage distortions in the representative setting from Figure~\ref{fig:mismatch-overview}. Each panel overlays Base, Power, and SoftSat on the same candidate pool. (a) SoftSat restores the gold-answer margin without discarding correct-answer mass. (b) SoftSat avoids Power's concentration on a narrow high-likelihood portion of the correct set. Full three-model diagnostics appear in the appendix.}
    \label{fig:mismatch-repair}
\end{figure*}

We do not repeat a dose panel for SoftSat. Its gain is defined on a bounded relative-rank coordinate, so fixed hyperparameters impose the same gain curve across problems, up to likelihood ties. The empirical question left for Figure~\ref{fig:mismatch-repair} is therefore whether this controlled dose preserves useful coverage.

\subsection{Checking the Weighted Realization}

On all 1,118 human-verified MoreHopQA Case-5 examples~\citep{schnitzler2024morehopqa}, repeated Power-SMC sampling and its weighted realization select the same normalized answer in 99.55\% of cases. Their accuracies differ by $0.089$ percentage points, with a paired-bootstrap 95\% confidence interval of $[-0.268,0.447]$, while mean inference time falls by $7.42\times$. Appendix~\ref{app:weighted-realization-check} gives the full protocol and distributional diagnostics. This check supports the weighted realization in this finite setting; it is not a claim of universal equivalence or a test of the asymptotic convergence rate.

\section{Conclusion}

We studied why sequence-level Power Sampling can improve the probability of correct trajectories yet degrade decisions formed from multiple samples. Our analysis identified two sources of this failure: coverage mismatch concentrates probability mass on a narrow set of dominant trajectories and answers, while dose mismatch makes a fixed exponent induce sharply different deformations across problems. This diagnosis led to Relative-Rank SoftSat, which replaces absolute likelihood with a problem-relative coordinate and saturates the gain assigned to high-probability trajectories. We also connected sampling from an ideal Power target to weighted consensus over a Base candidate pool, providing a practical way to combine sequence-level targets with multi-sample inference. Across models and reasoning benchmarks, standard Power produced large and inconsistent regressions, whereas SoftSat removed these failures while retaining gains in several settings; a finite-sample check further showed close agreement between repeated Power-SMC sampling and its weighted realization. More broadly, our results show that a sharpened distribution should be judged not only by the quality of individual trajectories or pass@$k$, but also by whether it preserves the support structure required by the downstream decision rule.

\ifdefined\ARXIVVERSION
\else
\subsection*{AI use statement}

Generative AI tools were used to assist with manuscript formatting, migration to the ICLR 2027 template, language editing, and verification and formatting of bibliographic metadata. The authors reviewed the AI-assisted changes against the underlying code, experimental records, and cited sources. No AI-generated output was treated as experimental evidence, and the authors take responsibility for the final content of this work.

\subsection*{Reproducibility statement}

Sections~\ref{sec:preliminaries}--\ref{sec:method} specify the target distributions, finite-pool weighting rules, and algorithms used in this work. The experimental setup, trajectory budgets, hyperparameters, dataset splits, and model identifiers are documented in the experiments, while the appendix provides complete proofs and extended diagnostic and budget-sweep results. The accompanying code artifact implements the finite-pool weighting and consensus readout used in the reported analyses.
\fi

\bibliography{references}

@misc{kaplan2020scaling,
  author        = {Kaplan, Jared and McCandlish, Sam and Henighan, Tom and Brown, Tom B. and Chess, Benjamin and Child, Rewon and Gray, Scott and Radford, Alec and Wu, Jeffrey and Amodei, Dario},
  title         = {Scaling Laws for Neural Language Models},
  year          = {2020},
  eprint        = {2001.08361},
  archivePrefix = {arXiv}
}

@inproceedings{lewkowycz2022quantitative,
  author    = {Lewkowycz, Aitor and Andreassen, Anders and Dohan, David and Dyer, Ethan and Michalewski, Henryk and Ramasesh, Vinay and Slone, Ambrose and Anil, Cem and Schlag, Imanol and Gutman-Solo, Theo and Wu, Yuhuai and Neyshabur, Behnam and Gur-Ari, Guy and Misra, Vedant},
  title     = {Solving Quantitative Reasoning Problems with Language Models},
  booktitle = {Advances in Neural Information Processing Systems},
  volume    = {35},
  year      = {2022}
}

@misc{chen2021codex,
  author        = {Chen, Mark and others},
  title         = {Evaluating Large Language Models Trained on Code},
  year          = {2021},
  eprint        = {2107.03374},
  archivePrefix = {arXiv}
}

@inproceedings{snell2025scaling,
  author    = {Snell, Charlie and Lee, Jaehoon and Xu, Kelvin and Kumar, Aviral},
  title     = {Scaling {LLM} Test-Time Compute Optimally Can Be More Effective Than Scaling Parameters for Reasoning},
  booktitle = {International Conference on Learning Representations},
  year      = {2025}
}

@inproceedings{wu2025inference,
  author    = {Wu, Yangzhen and Sun, Zhiqing and Li, Shanda and Welleck, Sean and Yang, Yiming},
  title     = {Inference Scaling Laws: An Empirical Analysis of Compute-Optimal Inference for {LLM} Problem-Solving},
  booktitle = {International Conference on Learning Representations},
  year      = {2025}
}

@inproceedings{huang2025sharpening,
  author    = {Huang, Audrey and Block, Adam and Foster, Dylan and Rohatgi, Dhruv and Zhang, Cyril and Simchowitz, Max and Ash, Jordan T. and Krishnamurthy, Akshay},
  title     = {Self-Improvement in Language Models: The Sharpening Mechanism},
  booktitle = {International Conference on Learning Representations},
  year      = {2025}
}

@inproceedings{wang2023selfconsistency,
  author    = {Wang, Xuezhi and Wei, Jason and Schuurmans, Dale and Le, Quoc V. and Chi, Ed H. and Narang, Sharan and Chowdhery, Aakanksha and Zhou, Denny},
  title     = {Self-Consistency Improves Chain of Thought Reasoning in Language Models},
  booktitle = {International Conference on Learning Representations},
  year      = {2023}
}

@inproceedings{yue2025reasoningcapacity,
  author    = {Yue, Yang and Chen, Zhiqi and Lu, Rui and Zhao, Andrew and Wang, Zhaokai and Yue, Yang and Song, Shiji and Huang, Gao},
  title     = {Does Reinforcement Learning Really Incentivize Reasoning Capacity in {LLM}s Beyond the Base Model?},
  booktitle = {Advances in Neural Information Processing Systems},
  year      = {2025}
}

@inproceedings{tan2025selfconsistent,
  author    = {Tan, Hexiang and Sun, Fei and Liu, Sha and Su, Du and Cao, Qi and Chen, Xin and Wang, Jingang and Cai, Xunliang and Wang, Yuanzhuo and Shen, Huawei and Cheng, Xueqi},
  title     = {Too Consistent to Detect: A Study of Self-Consistent Errors in {LLM}s},
  booktitle = {Proceedings of the 2025 Conference on Empirical Methods in Natural Language Processing},
  pages     = {4755--4765},
  year      = {2025},
  doi       = {10.18653/v1/2025.emnlp-main.238}
}

@inproceedings{li2025dynamic,
  author    = {Li, Yiwei and Zhang, Ji and Feng, Shaoxiong and Yuan, Peiwen and Wang, Xinglin and Shi, Jiayi and Zhang, Yueqi and Tan, Chuyi and Pan, Boyuan and Hu, Yao and Li, Kan},
  title     = {Revisiting Self-Consistency from Dynamic Distributional Alignment Perspective on Answer Aggregation},
  booktitle = {Findings of the Association for Computational Linguistics: {ACL} 2025},
  pages     = {25208--25223},
  year      = {2025},
  doi       = {10.18653/v1/2025.findings-acl.1293}
}

@inproceedings{wan2025rasc,
  author    = {Wan, Guangya and Wu, Yuqi and Chen, Jie and Li, Sheng},
  title     = {Reasoning Aware Self-Consistency: Leveraging Reasoning Paths for Efficient {LLM} Sampling},
  booktitle = {Proceedings of the 2025 Conference of the Nations of the Americas Chapter of the Association for Computational Linguistics: Human Language Technologies},
  pages     = {3613--3635},
  year      = {2025},
  doi       = {10.18653/v1/2025.naacl-long.184}
}

@inproceedings{karan2026reasoning,
  author    = {Karan, Aayush and Du, Yilun},
  title     = {Reasoning with Sampling: Your Base Model Is Smarter Than You Think},
  booktitle = {International Conference on Learning Representations},
  year      = {2026}
}

@inproceedings{ji2026scalable,
  author    = {Ji, Xiaotong and Tutunov, Rasul and Zimmer, Matthieu and Bou Ammar, Haitham},
  title     = {Scalable Power Sampling: Unlocking Efficient, Training-Free Reasoning for {LLM}s via Distribution Sharpening},
  booktitle = {International Conference on Machine Learning},
  year      = {2026}
}

@misc{azizi2026powersmc,
  author        = {Azizi, Seyedarmin and Baghaei Potraghloo, Erfan and Ahmadi, Minoo and Kundu, Souvik and Pedram, Massoud},
  title         = {{Power-SMC}: Low-Latency Sequence-Level Power Sampling for Training-Free {LLM} Reasoning},
  year          = {2026},
  eprint        = {2602.10273},
  archivePrefix = {arXiv}
}

@inproceedings{chen2023codet,
  author    = {Chen, Bei and Zhang, Fengji and Nguyen, Anh and Zan, Daoguang and Lin, Zeqi and Lou, Jian-Guang and Chen, Weizhu},
  title     = {{CodeT}: Code Generation with Generated Tests},
  booktitle = {International Conference on Learning Representations},
  year      = {2023}
}

@misc{chen2023universal,
  author        = {Chen, Xinyun and Aksitov, Renat and Alon, Uri and Ren, Jie and Xiao, Kefan and Yin, Pengcheng and Prakash, Sushant and Sutton, Charles and Wang, Xuezhi and Zhou, Denny},
  title         = {Universal Self-Consistency for Large Language Model Generation},
  year          = {2023},
  eprint        = {2311.17311},
  archivePrefix = {arXiv}
}

@inproceedings{choi2026modex,
  author    = {Choi, Hyeong Kyu and Li, Sharon},
  title     = {{ModeX}: Evaluator-Free Best-of-N Selection for Open-Ended Generation},
  booktitle = {Proceedings of the 64th Annual Meeting of the Association for Computational Linguistics},
  year      = {2026}
}

@inproceedings{taubenfeld2025confidence,
  author    = {Taubenfeld, Amir and Sheffer, Tom and Ofek, Eran and Feder, Amir and Goldstein, Ariel and Gekhman, Zorik and Yona, Gal},
  title     = {Confidence Improves Self-Consistency in {LLM}s},
  booktitle = {Findings of the Association for Computational Linguistics: {ACL} 2025},
  year      = {2025}
}

@inproceedings{zhou2025bridging,
  author    = {Zhou, Zhi and Tan, Yuhao and Li, Zenan and Yao, Yuan and Guo, Lan-Zhe and Li, Yu-Feng and Ma, Xiaoxing},
  title     = {A Theoretical Study on Bridging Internal Probability and Self-Consistency for {LLM} Reasoning},
  booktitle = {Advances in Neural Information Processing Systems},
  year      = {2025}
}

@inproceedings{chen2026modelswitch,
  author    = {Chen, Jianhao and Xun, Zishuo and Zhou, Bocheng and Qi, Han and Zhang, Hangfan and Zhang, Qiaosheng and Chen, Yang and Hu, Wei and Qu, Yuzhong and Hu, Shuyue},
  title     = {Do We Truly Need So Many Samples? Multi-{LLM} Repeated Sampling Efficiently Scales Test-Time Compute},
  booktitle = {Proceedings of the AAAI Conference on Artificial Intelligence},
  volume    = {40},
  pages     = {20083--20091},
  year      = {2026},
  doi       = {10.1609/aaai.v40i24.39094}
}

@inproceedings{oh2026latent,
  author    = {Oh, Jungsuk and Lee, Jay-Yoon},
  title     = {Latent Self-Consistency for Reliable Majority-Set Selection in Short- and Long-Answer Reasoning},
  booktitle = {Proceedings of the AAAI Conference on Artificial Intelligence},
  volume    = {40},
  pages     = {32591--32599},
  year      = {2026},
  doi       = {10.1609/aaai.v40i38.40536}
}

@inproceedings{wang2026optscale,
  author    = {Wang, Youkang and Wang, Jian and Chen, Rubing and Wei, Xiao-Yong},
  title     = {{OptScale}: Probabilistic Optimality for Inference-Time Scaling},
  booktitle = {Proceedings of the AAAI Conference on Artificial Intelligence},
  volume    = {40},
  pages     = {33710--33718},
  year      = {2026},
  doi       = {10.1609/aaai.v40i40.40661}
}

@inproceedings{khanh2026diverse,
  author    = {Khanh, Ly Tran Ho and Zhu, Dongxuan and Yue, Man-Chung and Nguyen, Viet Anh},
  title     = {Test-Time Diverse Reasoning by Riemannian Activation Steering},
  booktitle = {Proceedings of the AAAI Conference on Artificial Intelligence},
  volume    = {40},
  pages     = {31429--31437},
  year      = {2026},
  doi       = {10.1609/aaai.v40i37.40407}
}

@inproceedings{zhuo2025bigcodebench,
  author    = {Zhuo, Terry Yue and Vu, Minh Chien and Chim, Jenny and Hu, Han and Yu, Wenhao and Widyasari, Ratnadira and Yusuf, Imam Nur Bani and Zhan, Haolan and He, Junda and Paul, Indraneil and Brunner, Simon and Gong, Chen and Hoang, James and Zebaze, Armel and Hong, Xiaoheng and Li, Wen-Ding and Kaddour, Jean and Xu, Ming and Zhang, Zhihan and Yadav, Prateek and Jain, Naman and Gu, Alex and Cheng, Zhoujun and Liu, Jiawei and Liu, Qian and Wang, Zijian and Hui, Binyuan and Muennighoff, Niklas and Lo, David and Fried, Daniel and Du, Xiaoning and de Vries, Harm and von Werra, Leandro},
  title     = {{BigCodeBench}: Benchmarking Code Generation with Diverse Function Calls and Complex Instructions},
  booktitle = {International Conference on Learning Representations},
  year      = {2025},
  url       = {https://openreview.net/forum?id=YrycTjllL0}
}

@inproceedings{mahdavi2025liveaops,
  author    = {Mahdavi, Sadegh and Li, Muchen and Liu, Kaiwen and Thrampoulidis, Christos and Sigal, Leonid and Liao, Renjie},
  title     = {Leveraging Online Olympiad-Level Math Problems for {LLM}s Training and Contamination-Resistant Evaluation},
  booktitle = {Proceedings of the 42nd International Conference on Machine Learning},
  series    = {Proceedings of Machine Learning Research},
  volume    = {267},
  year      = {2025},
  url       = {https://proceedings.mlr.press/v267/mahdavi25a.html}
}

@inproceedings{feng2025physics,
  author    = {Feng, Kaiyue and Zhao, Yilun and Liu, Yixin and Yang, Tianyu and Zhao, Chen and Sous, John and Cohan, Arman},
  title     = {{PHYSICS}: Benchmarking Foundation Models on University-Level Physics Problem Solving},
  booktitle = {Findings of the Association for Computational Linguistics: ACL 2025},
  pages     = {11717--11743},
  year      = {2025},
  doi       = {10.18653/v1/2025.findings-acl.610},
  url       = {https://aclanthology.org/2025.findings-acl.610/}
}

@article{schnitzler2024morehopqa,
  author        = {Schnitzler, Julian and Ho, Xanh and Huang, Jiahao and Boudin, Florian and Sugawara, Saku and Aizawa, Akiko},
  title         = {{MoreHopQA}: More Than Multi-hop Reasoning},
  journal       = {arXiv preprint arXiv:2406.13397},
  year          = {2024},
  doi           = {10.48550/arXiv.2406.13397},
  url           = {https://arxiv.org/abs/2406.13397}
}
\ifdefined\ARXIVVERSION
\bibliographystyle{plainnat}
\else
\bibliographystyle{iclr2027_conference}
\fi

\appendix

\section{Full Mismatch Diagnostics}
\label{app:full-mismatch-diagnostics}

Figure~\ref{fig:full-power-mismatch} extends the three diagnostics in Figure~\ref{fig:mismatch-overview} to every evaluated model on LiveAoPSBench. Figure~\ref{fig:full-softsat-repair} similarly extends the coverage comparison in Figure~\ref{fig:mismatch-repair}.

\begin{figure*}[t]
    \centering
    \includegraphics[width=0.98\textwidth]{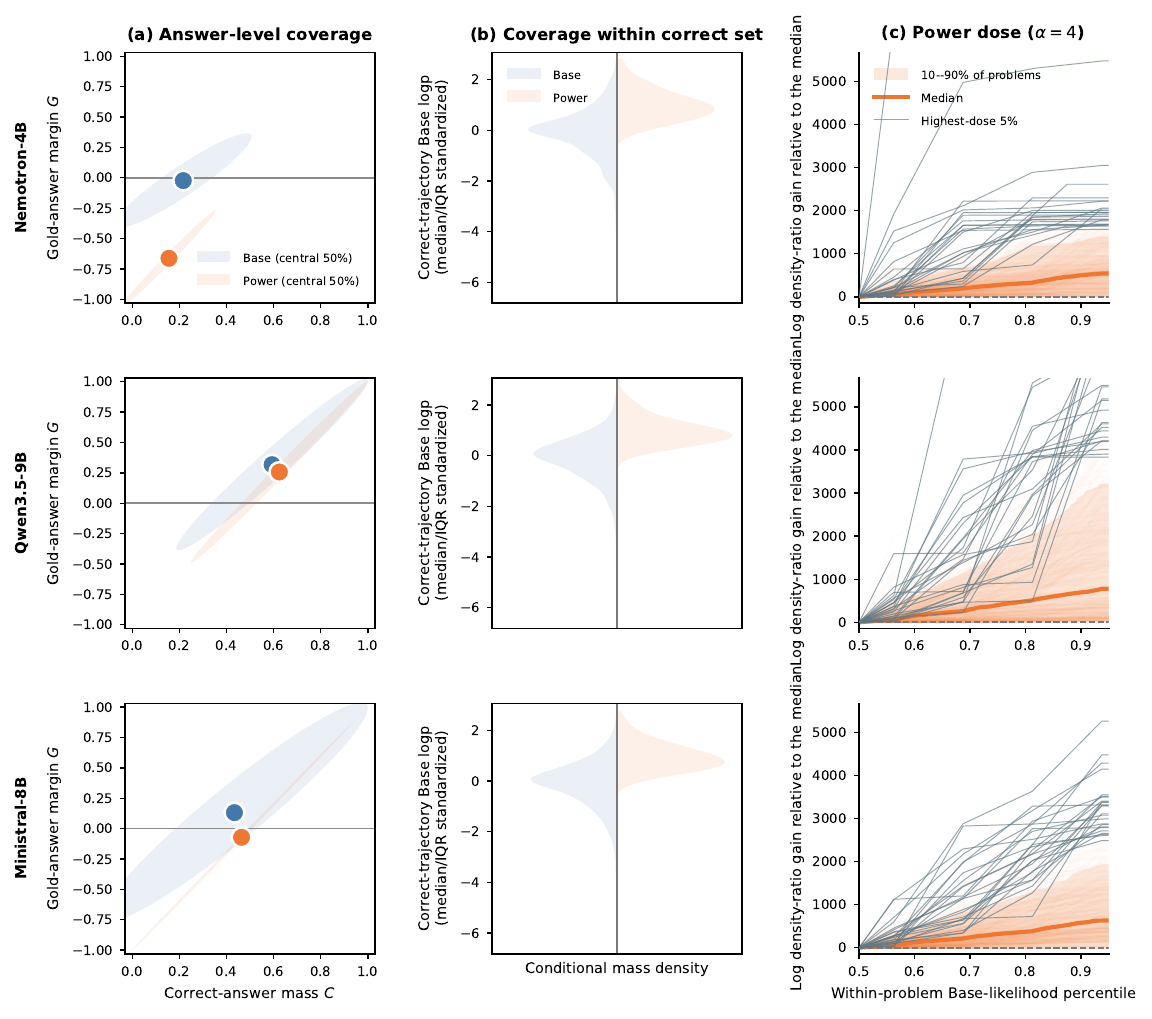}
    \caption{Full Power mismatch diagnostics on LiveAoPSBench. Rows correspond to Nemotron-4B, Qwen3.5-9B, and Ministral-8B; columns follow Figure~\ref{fig:mismatch-overview}.}
    \label{fig:full-power-mismatch}
\end{figure*}

\begin{figure*}[t]
    \centering
    \includegraphics[width=0.94\textwidth]{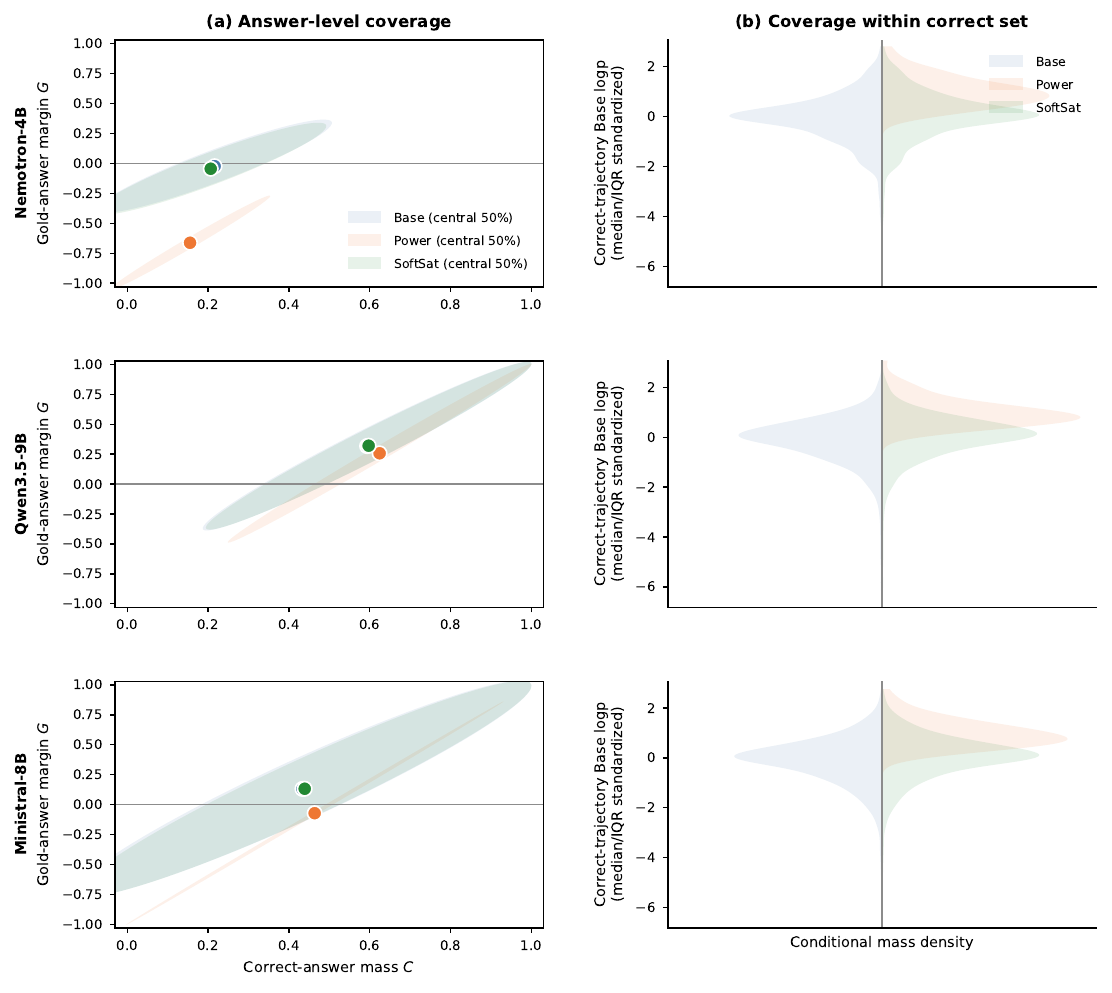}
    \caption{Full coverage comparison among Base, Power, and SoftSat on LiveAoPSBench. Rows correspond to Nemotron-4B, Qwen3.5-9B, and Ministral-8B; columns follow Figure~\ref{fig:mismatch-repair}.}
    \label{fig:full-softsat-repair}
\end{figure*}

\section{Budget Sweep Across Benchmarks}
\label{app:budget-sweep}

Figure~\ref{fig:full-budget-beta} extends Figure~\ref{fig:budget-beta} to the three evaluated benchmarks.

\begin{figure*}[t]
    \centering
    \includegraphics[width=0.99\textwidth]{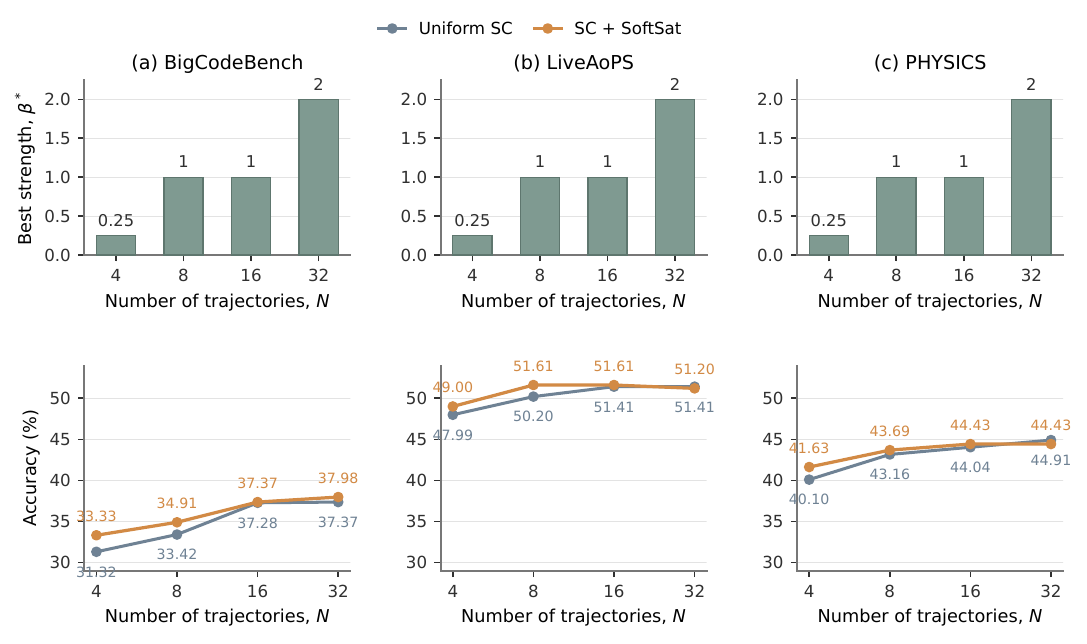}
    \caption{Budget-dependent SoftSat behavior for Ministral across BigCodeBench, LiveAoPSBench, and PHYSICS. Within each benchmark, the upper panel shows the best $\beta$ at each budget and the lower panel compares uniform $\mathrm{SC}_N$ with SoftSat-weighted $\mathrm{SC}_N$.}
    \label{fig:full-budget-beta}
\end{figure*}

\section{Finite-Sample Check of the Weighted Realization}
\label{app:weighted-realization-check}

The preceding lemmas motivate a weighted realization of target answer mass, but they do not imply identical decisions for a finite particle system. We therefore perform a narrow check on all 1,118 human-verified MoreHopQA Case-5 examples~\citep{schnitzler2024morehopqa}. With Qwen3.5-9B and $\alpha=4$, the sampling reference runs an $M=8$ particle Power-SMC system independently $N=8$ times and applies uniform SC to the eight returned answers. The weighted realization instead runs Power-SMC once and aggregates the answer mass of its $M=8$ terminal particles using their normalized weights.

\begin{figure*}[t]
    \centering
    \includegraphics[width=0.92\textwidth]{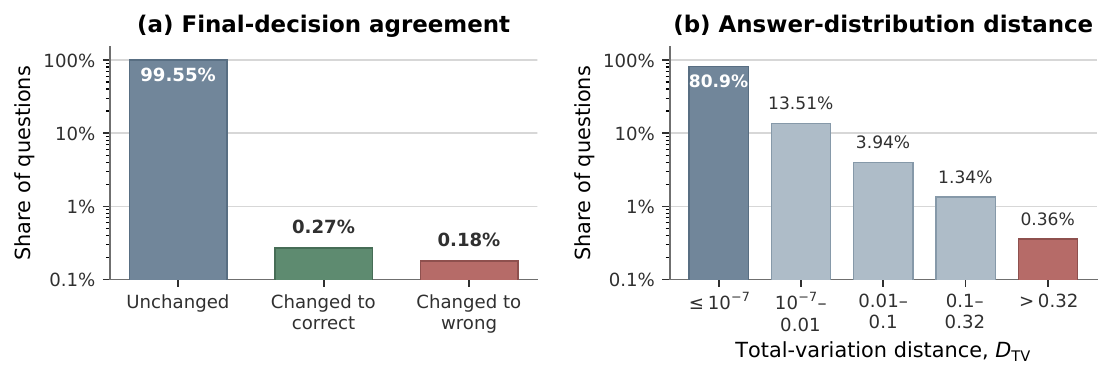}
    \caption{Finite-sample agreement between repeated Power-SMC sampling and its weighted realization on MoreHopQA Case-5. (a) The two procedures return the same normalized answer on 99.55\% of questions; the remaining changes are split by whether weighting corrects or introduces an error. (b) Most answer distributions are also close in total variation distance (DTV), although a small tail remains.}
    \label{fig:powersmc-weighted-consistency}
\end{figure*}

Figure~\ref{fig:powersmc-weighted-consistency} shows that 1,113 of 1,118 final decisions are unchanged. Weighting corrects three errors and introduces two, yielding 91.14\% accuracy versus 91.06\% for repeated sampling, a difference of $0.089$ percentage points with a paired-bootstrap 95\% confidence interval of $[-0.268,0.447]$. The underlying answer distributions agree nearly as closely: 80.86\% have DTV at most $10^{-7}$ and 94.36\% have DTV at most $0.01$, while only 0.36\% exceed $10^{-0.5}$. Mean inference time falls from 19.50 to 2.63 seconds per question, a $7.42\times$ reduction.

\section{Proofs for Weighted Target Approximation}
\label{app:weighted-proofs}

\begin{proof}[Proof of Lemma~\ref{lem:power-expectation-equivalence}]
Fix an answer $a$ and write $f_a(\tau)=\mathbf{1}\{A(\tau)=a\}$. Define
\begin{equation}
\overline U_N=\frac{1}{N}\sum_{i=1}^{N}r_\alpha(X_i)f_a(X_i),
\qquad
\overline R_N=\frac{1}{N}\sum_{i=1}^{N}r_\alpha(X_i).
\end{equation}
Since $q_\alpha(\tau)=p(\tau)r_\alpha(\tau)/Z_\alpha$,
\begin{equation}
\mathbb{E}_p[r_\alpha(X)f_a(X)]
=Z_\alpha Q_{q_\alpha}(a),
\qquad
\mathbb{E}_p[r_\alpha(X)]=Z_\alpha.
\end{equation}
The strong law of large numbers therefore gives
$\overline U_N\to Z_\alpha Q_{q_\alpha}(a)$ and
$\overline R_N\to Z_\alpha$ almost surely. Because $Z_\alpha>0$, their ratio converges almost surely to $Q_{q_\alpha}(a)$. Moreover,
$0\leq\overline U_N/\overline R_N\leq1$, so bounded convergence yields
\begin{equation}
\lim_{N\to\infty}\mathbb{E}\!\left[
\widehat Q_N^{\mathrm{weight}}(a)
\right]=Q_{q_\alpha}(a).
\end{equation}
For target samples, $f_a(Y_i)$ is Bernoulli with mean $Q_{q_\alpha}(a)$. Hence
$\mathbb{E}[\widehat Q_N^{\mathrm{sample}}(a)]=Q_{q_\alpha}(a)$ for every $N$, proving the result.
\end{proof}

\begin{proof}[Proof of Lemma~\ref{lem:power-estimation-error}]
Let $\rho(\tau)=q_\alpha(\tau)/p(\tau)$ and $\theta=Q_{q_\alpha}(a)$. Multiplying all weights by the same positive constant does not change the self-normalized estimator, so
\begin{equation}
\begin{aligned}
\widehat Q_N^{\mathrm{weight}}(a)-\theta
&=\frac{\overline H_N}{\overline\rho_N},\\
\overline H_N
&=\frac{1}{N}\sum_{i=1}^{N}\rho(X_i)(f_a(X_i)-\theta),\\
\overline\rho_N
&=\frac{1}{N}\sum_{i=1}^{N}\rho(X_i).
\end{aligned}
\label{eq:proof-self-normalized-error}
\end{equation}
Here $\mathbb{E}_p[\rho(X)]=1$ and
$\mathbb{E}_p[\rho(X)(f_a(X)-\theta)]=0$. Put
$\sigma_a^2=\operatorname{Var}_p[\rho(X)(f_a(X)-\theta)]$, which is finite because $0<\rho\leq C$.

Consider the event $E_N=\{\overline\rho_N\geq1/2\}$. On $E_N$, \eqref{eq:proof-self-normalized-error} implies
\begin{equation}
\left(\widehat Q_N^{\mathrm{weight}}(a)-\theta\right)^2
\leq4\overline H_N^2,
\qquad
\mathbb{E}[\overline H_N^2]=\frac{\sigma_a^2}{N}.
\end{equation}
On $E_N^c$, the squared error is at most one. Since $0\leq\rho(X)\leq C$ and $\mathbb{E}[\rho(X)]=1$, Hoeffding's inequality gives
\begin{equation}
\Pr(E_N^c)
=\Pr(\overline\rho_N-1<-1/2)
\leq\exp\!\left(-\frac{N}{2C^2}\right).
\end{equation}
Combining the two events,
\begin{equation}
\mathbb{E}\!\left[
\left(\widehat Q_N^{\mathrm{weight}}(a)-\theta\right)^2
\right]
\leq\frac{4\sigma_a^2}{N}
+\exp\!\left(-\frac{N}{2C^2}\right).
\label{eq:weighted-mse-bound}
\end{equation}

The target-sampling estimator is unbiased with variance
$\theta(1-\theta)/N$. It is independent of the Base-pool estimator, so the cross term between their centered errors vanishes. Writing
$D_N(a)=\widehat Q_N^{\mathrm{weight}}(a)-\widehat Q_N^{\mathrm{sample}}(a)$,
\begin{equation}
\begin{aligned}
\mathbb{E}[D_N(a)^2]
&\leq
\frac{4\sigma_a^2+\theta(1-\theta)}{N}
+\exp\!\left(-\frac{N}{2C^2}\right)\\
&=O(N^{-1}).
\end{aligned}
\end{equation}
The $O_p(N^{-1/2})$ statement follows immediately from this second-moment bound and Markov's inequality.
\end{proof}

Finally, suppose the answer space is finite and $a^\star$ is the unique mode of $Q_{q_\alpha}$ with margin
$\delta=Q_{q_\alpha}(a^\star)-\max_{a\neq a^\star}Q_{q_\alpha}(a)>0$.
Both estimators converge in probability to $Q_{q_\alpha}(a)$ for every answer. A union bound over the finite answer space shows that, with probability tending to one, all estimated masses are within $\delta/3$ of their targets. On this event, both estimators select $a^\star$, establishing the mode-agreement consequence stated in the main text.

\end{document}